\documentclass[conference]{IEEEtran}

\IEEEoverridecommandlockouts

\usepackage{fca}
\usepackage{algorithm}
\usepackage{listings}

\usepackage[utf8]{inputenc}
\usepackage{pgfplots}
\DeclareUnicodeCharacter{2212}{−}
\usepgfplotslibrary{groupplots,dateplot}
\usetikzlibrary{patterns,shapes.arrows}
\pgfplotsset{compat=newest}

\usepackage{tikz}
\usepackage{tikzscale}

\usepackage{cite}
\usepackage{amsmath,amssymb,amsfonts, mathtools, amsthm}
\usepackage[hidelinks]{hyperref}
\usepackage{cleveref}
\usepackage{algorithmic}
\usepackage{graphicx}
\usepackage{textcomp}
\usepackage{xcolor}
\usepackage{booktabs}
\def\BibTeX{{\rm B\kern-.05em{\sc i\kern-.025em b}\kern-.08em
    T\kern-.1667em\lower.7ex\hbox{E}\kern-.125emX}}

\newtheorem{theorem}{Theorem}[section]
\newtheorem{corollary}{Corollary}[theorem]

\newtheorem{lemma}[theorem]{Lemma}

\theoremstyle{definition}
\newtheorem{definition}{Definition}[section]
\theoremstyle{remark}

\newif\ifentwurf

\ifentwurf
\usepackage{xcolor}

\newcommand{\quest}[1]{\noindent{\bf\color{blue} Quest : #1}\newline}
\else
\newcommand{\quest}[1]{}
\fi
\let\subset\subseteq
\let\cref\Cref

\begin{document}

\title{Greedy Discovery of Ordinal Factors}

\author{\IEEEauthorblockN{Dominik Dürrschnabel}
  \IEEEauthorblockA{\textit{Knowledge and Data Engineering Group} and \\
    \textit{Interdisciplinary Research Center }\\
    \textit{for Information System Design}\\
    \textit{University of Kassel}\\
    Kassel, Germany \\
    duerrschnabel@cs.uni-kassel.de\\
    ORCID: 0000-0002-0855-4185}
  \and
  \IEEEauthorblockN{Gerd Stumme}
  \IEEEauthorblockA{\textit{Knowledge and Data Engineering Group} and \\
    \textit{Interdisciplinary Research Center}\\
    \textit{for Information System Design}\\
    \textit{University of Kassel}\\
    Kassel, Germany \\
    stumme@cs.uni-kassel.de\\
    ORCID: 0000-0002-0570-7908}
}

\maketitle

\begin{abstract}
  In large datasets, it is hard to discover and analyze structure.
  It is thus common to introduce tags or keywords for the items.
  In applications, such datasets are then filtered based on these tags.
  Still, even medium-sized datasets with a few tags result in complex and for humans hard-to-navigate systems.
  In this work, we adopt the method of ordinal factor analysis to address this problem.
  An ordinal factor arranges a subset of the tags in a linear order based on their underlying structure.
  A complete ordinal factorization, which consists of such ordinal factors, precisely represents the original dataset.
  Based on such an ordinal factorization, we provide a way to discover and explain relationships between different items and attributes in the dataset.
  However, computing even just one ordinal factor of high cardinality is computationally complex.
  We thus propose the greedy algorithm \textsc{OrdiFIND} in this work.
  This algorithm extracts ordinal factors using already existing fast algorithms developed in formal concept analysis.
  Then, we leverage \textsc{OrdiFIND} to propose a comprehensive way to discover relationships in the dataset.
  We furthermore introduce a distance measure based on the representation emerging from the ordinal factorization to discover similar items.
  To evaluate the method, we conduct a case study on different datasets.
\end{abstract}


\section{Introduction}
\label{sec:introduction}
A binary dataset consists of an object set and an attribute set together with an incidence relation.
Items with tags, keywords, or categories are an example of such binary datasets where the items are the objects, and the tags are the attributes.
Typical examples of tagged data are movies or songs with category or genre tags.
Furthermore, keywords in articles such as blog posts or scientific works are also tagged datasets.
Another example of such data would be researchers publishing at different conferences.
Understanding relationships and extracting knowledge from such datasets based on these tags or keywords quickly becomes an impossible task, while on the other hand being a necessity.
Thus, providing systems that allow extracting explanations and discover relationships in such data is an important research task.

A common way to tackle this task is to treat the binary attributes as numerical, where the value 1 is assigned to an object if it has the attribute and 0 otherwise.
Then, methods from the classical toolkit of dimensionality reduction, such as principal component analysis, are applied.
These approaches merge the different tags into a few axes while weighting the attributes in each axis.
An emerging axis thus yields information about the presence or absence of correlated features in the original dataset.
Then, each object is assigned a real-valued number in each axis to represent whether this item has its attributes.
As a single axis represents multiple merged attributes, the resulting placement of objects with only a part of attributes yields an ambiguous representation.
Because of this, the main issues of this method arise.
Assigning a real value to an object is not consistent with the level of measurement of the underlying binary data.
The assigned value promotes the perception that an element has, compared to others, a stronger bond to some attributes, which is impossible.
Thus, a method that encourages such comparisons and results in such an inaccurate representation of the original information is, in our opinion, not valid.
An example of such a PCA projection is on the right side of \cref{fig:3d}.
\begin{figure*}[ht]
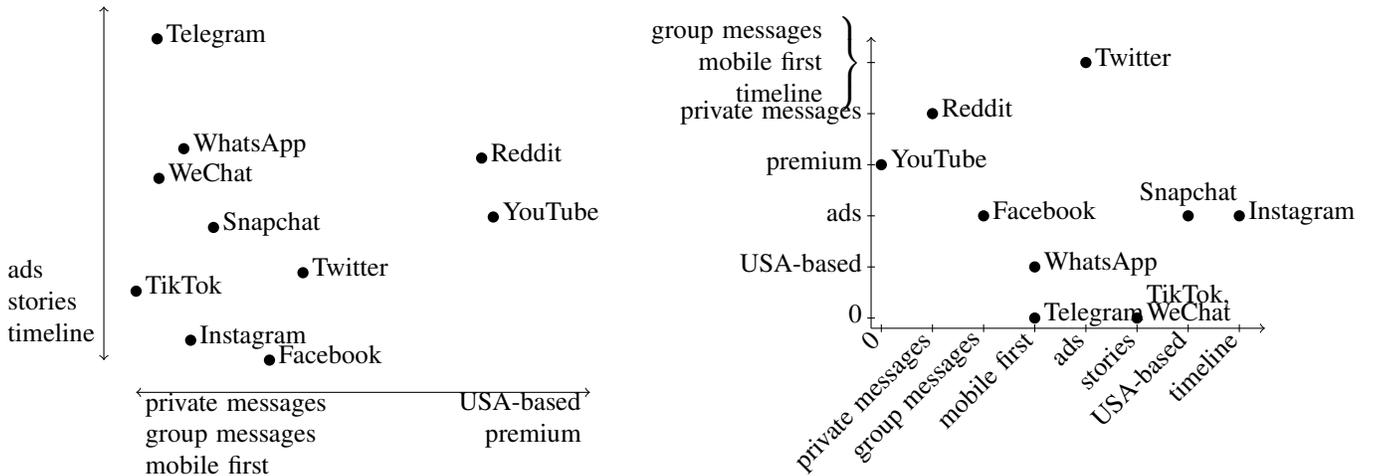

  \centering
  \includegraphics[height=18em]{tikz/pca.tikz}%
  \hfill
  \includegraphics[height=18em]{tikz/2d.tikz}%
  \caption{Left: A 2-dimensional embedding of the objects using principal component analysis. Right: An ordinal factorization restricted to the two largest factors for improved readability. All incidences can be deduced from the right projection except: (TikTok, timeline), (Whatsapp, stories), (Facebook, timeline), (YouTube, stories), (Facebook, stories). The ordinal projection does not contain false data.}
  \label{fig:3d}
\end{figure*}

To tackle this problem, Ganter and Glodeanu propose in~\cite{Ganter.2012} the \emph{ordinal factor analysis}, a method that does not use real-valued measurement on the binary attributes.
The main idea is to condense multiple attributes into a single factor, similar to principal component analysis.
The projection consists of linear orders of attributes, the so-called \emph{ordinal factors}.
Then, the method assigns each object, based on its attributes, a position in every factor.
Compared to the PCA approach, the positions assigned in the process are natural numbers instead of real-valued ones.
Thus, the resulting projection does not express inaccurate or even incorrect information.
A complete ordinal factorization furthermore allows deducing all original information.
Similar to the PCA projection, there is a visualization method for small datasets.
Thereby one places each object in a two-dimensional coordinate system at the last position for each axis such that it has all attributes until this position.
\cref{fig:3d} compares such a visualization to a PCA-projection.
The main advantage of the ordinal factor plot compared to the principal component analysis is that all its information is guaranteed to be correct.
For example, Facebook and Instagram both have ads and are USA-based, which is depicted in the vertical axis.
On the other hand, the principal component analysis projection gives the impression that Instagram has more ads than Snapchat and both of them are not USA-based, which are both false claims.

However, Ganter and Glodeanu neither provide a method for the computation of ordinal factorizations nor does the visualization method transfer to larger datasets.
Yet, in our opinion, this method is especially suited to handle larger datasets.
Still, for those, neither a method to compute the factorizations nor a visualization method exists so far.
Furthermore, for the application of the method, it is still necessary to demonstrate how to use it to discover and explain structure in such large datasets.
Here is the point where we step in with our present work.
We propose \textsc{OrdiFIND}, a greedy algorithm to compute a complete ordinal factorization of a dataset leveraging fast algorithms from formal concept analysis.
Furthermore, we extend the existing theory of ordinal factor analysis into a tool to explain relationships in large binary datasets.
Then we demonstrate how to discover knowledge in them using ordinal factor analysis by investigating different datasets.
Thus, our research question condenses to the following:
\textbf{Is it possible to compute ordinal factorizations of large datasets and use them to discover relationships in the data?}

\section{Related Work}
\label{sec:related-work}

We discuss work related to our topic in three different domains.
First, we discuss general tagging systems that generate binary data.
Then, we consider methods that represent attributes in a low number of dimensions.
In the final section, we discuss work done in formal concept analysis, which is the realm in which our work is positioned.

\subsection{Tagging, Categories and Keywords}

In the modern world, data often appear in large, unorganized chunks.
Thus, tags, categories, or keywords are assigned to items to tackle the problem of the sheer amount of data.
These allow the items to be filtered.
The transition between those concepts is seamless, as they are often used interchangeably.
Frequently, a community defines the tags on the items, which is a process commonly referred to as social tagging.
An extensive survey on social tagging techniques is published in~\cite{Gupta.2010}.
\cite{Nov.2008}~is a study examining the reasons behind tagging.
In \cite{Heckner.2008}, various types of tags created by users are analyzed and compared.
A method to automatically extract topics from a paper corpus is described in \cite{Onan.2019}
The automatic annotation of images in the domain of computer vision is explored in \cite{Kalayeh.2014}.

\subsection{Methods to Embed (Binary) Data}
\label{sec:other-methods-embed}

The two main reasons for the research of embedding high dimensional data into a lower number of dimensions are the following.
First, embeddings are applied to escape the so-called curse of dimensionality~\cite{Bellman.1966}, which is a phenomenon that causes distances to become less meaningful in higher dimensional spaces.
Secondly, it provides a way to better visualize~\cite{Hoffman.2001} complex data, as humans are better adapted to understand relationships in lower-dimensional spaces.
For an extensive survey on dimensional reduction methods, we refer the reader to~\cite{Espadoto.2021}.
A commonly applied approach for embedding high dimensional data into lower dimensions is the principal component analysis~\cite{Pearson.1901}, which is a method that minimizes the average squared distances from the data points to a line.
It is often confused~\cite{Jain.2013} with exploratory factor analysis~\cite{Child.1990}, a technique that allows exploring a dataset by reconstructing underlying factors.

\subsection{Factor Analysis in Formal Concept Analysis}
\label{sec:form-conc-analys}

Formal concept analysis is a research field of mathematical data analysis introduced by Wille in 1982~\cite{Wille.1982}.
For a comprehensive overview of the theory and foundations of formal concept analysis, we refer the reader to \cite{fca-book}.
The canonical way to depict information in formal concept analysis is the concept lattice.
Still, some approaches arrange data differently to make it more accessible.
One of them is the Boolean factor analysis, which was developed in~\cite{Keprt.2004, Keprt.2006, Belohlavek.2007, Belohlavek.2010, Boeck.1988}.
In~\cite{Ganter.2012} introduces the ordinal factor analysis.
Applications of the method are demonstrated in~\cite{Glodeanu.2013} where it is applied to some smaller medical datasets.
In \cite{Glodeanu.2013_2} the theory is lifted to the case of triadic incidence relations and in \cite{Glodeanu.2014} into the setting of fuzzy formal contexts.

\section{Basics for Ordinal Factor Analysis}
\label{sec:previous-work}

As formal concept analysis is the realm of ordinal factor analysis, we start this section with a brief introduction to FCA.
Then, we introduce the Boolean factor analysis, which gives rise to the ordinal factor analysis.
The running example that we use to explain the theory of this section is the dataset in \cref{fig:dataset} about different features and attributes of various social media platforms.
\begin{figure}[b]
  \centering
  \begin{cxt}
    \atr{USA-based}
    \atr{premium}
    \atr{ads}
    \atr{private messages}
    \atr{group messages}
    \atr{mobile first}
    \atr{stories}
    \atr{timeline}
    \obj{x.xxx.xx}{Facebook}
    \obj{x.xxxxxx}{Instagram}
    \obj{xxxx....}{Reddit}
    \obj{x.xxxxx.}{Snapchat}
    \obj{...xxx..}{Telegram}
    \obj{..xxxxxx}{TikTok}
    \obj{xxxxxx.x}{Twitter}
    \obj{..xxxxx.}{WeChat}
    \obj{x..xxxx.}{WhatsApp}
    \obj{xxx...x.}{YouTube}
  \end{cxt}
  \caption{Running example: This dataset compares attributes of different social media platforms.}
  \label{fig:dataset}
\end{figure}

\subsection{Formal Concept Analysis}
\label{sec:form-conc-analys-1}

The notations used in this work are the standard ones notations adapted from~\cite{fca-book}.
A \emph{formal context} is a triple $(G,M,I)$ consisting of an \emph{object set} $G$ and an \emph{attribute set} $M$ together with an incidence relation $I\subseteq G \times M$.
For arbitrary subsets $A\subset G$ and $B\subset M$ define the two \emph{derivation operators} as
\begin{align*}
  A' & \coloneqq \{m \in M \mid \forall g \in A: (g,m)\in I\}, \\
  B' & \coloneqq \{g \in G \mid \forall m \in B: (g,m)\in I\}.
\end{align*}
For a single attribute $m$ or object $g$, we write short $m'\coloneqq \{m\}'$ and $g'\coloneqq \{g\}'$.
A pair $(A,B)$ with $A\subset G$ and $B\subset M$ such that $A'=B$ and $B'=A$ is called a \emph{formal concept} of $(G,M,I)$.
The partial order relation
\begin{align*}
  (A_1,B_1)\leq (A_2,B_2):\iff A_1 \subseteq A_2\quad (\iff B_1 \supseteq B_2).
\end{align*}
orders all formal concepts of a context.
The set of all concepts with this order relation is denoted by $(\mathfrak{B},\leq)$.
A subset of concepts where all concepts are comparable with each other is called a \emph{chain}.
The order has a \emph{lattice} structure, i.e., for each pair of concepts, a unique supremum and a unique infimum exists.
The emerging so-called \emph{concept lattice} is used in formal concept analysis to analyze the structure and relationships in a formal context.
Two distinct concepts $(A_1,B_1)$ and $(A_2,B_2)$ are in \emph{covering relation}, if $(A_1,B_1) \leq (A_2,B_2)$ and there is no other concept $(A_3,B_3)$ with $(A_1,B_1) \leq (A_3,B_3) \leq (A_2,B_2)$.
It is then denoted by $(A_1,B_1) \prec (A_2,B_2)$.
An \emph{order diagram} graphically visualize the concept lattice.
In this diagram, nodes on the plane depict the concepts, whereby they are places such that smaller concept nodes are below the greater ones.
Continuous straight lines connect concept pairs in covering relation.
The greatest concept to contain a specific attribute is annotated with its name, while symmetrically, the smallest one to include an object has the annotation of this object name.
The attributes are annotated above the concept dot, while object labels are below the node.
By this construction, it is possible to reconstruct all information of the formal context from the concept lattice diagram.
\cref{fig:lattice} depicts the concept lattice for our running example.
\begin{figure}[t]
  \centering
  \includegraphics[width=\linewidth]{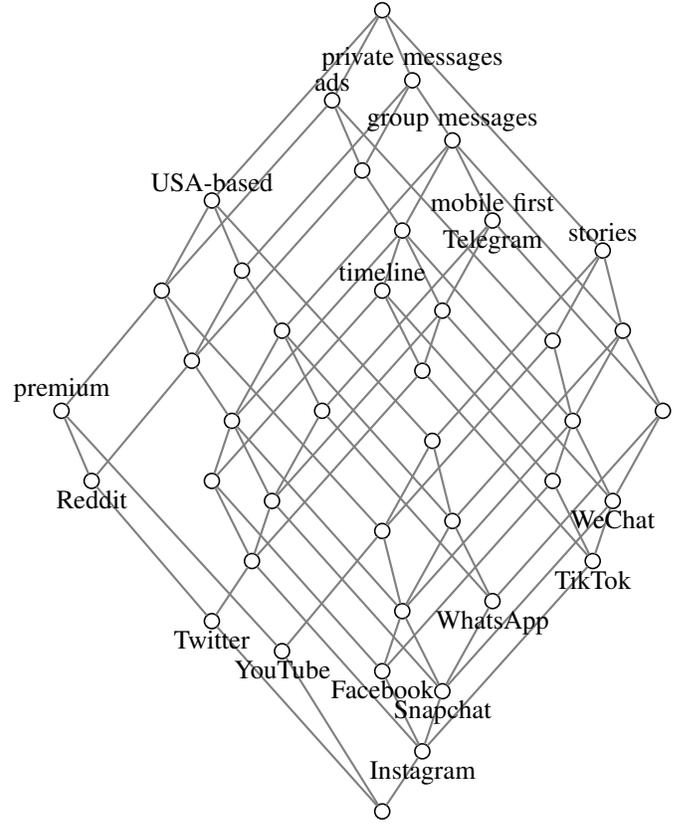}
  \caption{The order diagram of the concept lattice for the running example depicts its 41 concepts.}
  \label{fig:lattice}
\end{figure}

\subsection{Boolean Factor Analysis}
\label{sec:bool-fact-analys}
\begin{figure*}[!ht]
  \begin{minipage}[t]{0.7\linewidth}
    \strut\vspace*{-\baselineskip}\newline
    \scalebox{0.85}{
      \vbox{
        \begin{cxt}
          \att{f1}
          \att{f2}
          \att{f3}
          \att{f4}
          \att{f5}
          \att{f6}
          \att{f7}
          \att{f8}
          \att{f9}
          \att{f10}
          \att{f11}
          \att{f12}
          \att{f13}
          \att{f14}
          \att{f15}
          \att{f16}
          \att{f17}
          \obj{...x..x.x.xx.xxxx}{1}
          \obj{x.xx..xxxxxxxxxxx}{2}
          \obj{....xx.....x..x.x}{3}
          \obj{..x....xxxxxxxxxx}{4}
          \obj{............x..xx}{5}
          \obj{......xxxxx.xx.xx}{6}
          \obj{.x..xx...x.xx.xxx}{7}
          \obj{.......xxxx.xx.xx}{8}
          \obj{............xxxxx}{9}
          \obj{.....x....xx.xx..}{10}
        \end{cxt}}
    }

    \bigskip

    \begin{tabular}{ll}
      Objects:    & (1) Facebook
      (2) Instagram
      (3) Reddit
      (4) Snapchat
      (5) Telegram                \\&
      (6) TikTok
      (7) Twitter
      (8) WeChat
      (9) WhatsApp
      (10) YouTube
      \bigskip                    \\
      Attributes: & (a) USA-based
      (b) premium
      (c) ads
      (d) private messages        \\&
      (e) group messages
      (f) mobile first
      (g) stories
      (h) timeline
    \end{tabular}

  \end{minipage}%
  \hfill
  \begin{minipage}[t]{0.3\linewidth}
    \strut\vspace*{-\baselineskip}\newline
    \scalebox{0.85}{
      \vbox{
        \begin{cxt}
          \att{a}
          \att{b}
          \att{c}
          \att{d}
          \att{e}
          \att{f}
          \att{g}
          \att{h}
          \obj{x.xxxxxx}{f1}
          \obj{xxxxxx.x}{f2}
          \obj{x.xxxxx.}{f3}
          \obj{x.xxx.xx}{f4}
          \obj{xxxx....}{f5}
          \obj{xxx.....}{f6}
          \obj{..xxx.xx}{f7}
          \obj{..xxxxx.}{f8}
          \obj{..xxx.x.}{f9}
          \obj{..xxxx..}{f10}
          \obj{..x...x.}{f11}
          \obj{x.x.....}{f12}
          \obj{...xxx..}{f13}
          \obj{......x.}{f14}
          \obj{x.......}{f15}
          \obj{...xx...}{f16}
          \obj{...x....}{f17}
        \end{cxt}}}
  \end{minipage}
  \caption{A conceptual Boolean factorization of the dataset from \cref{fig:dataset} using 17 Boolean factors.}
  \label{fig:boolean_faktor}
\end{figure*}
Research on Boolean factor analysis is conducted within~\cite{Keprt.2004, Keprt.2006, Belohlavek.2007, Belohlavek.2010, Boeck.1988} and outside~\cite{Zhang.2007} the realm of formal concept analysis.
For this work, we use the notations as they are used in formal concept analysis.
A formal context $\context=\GMI$ can be disassembled into two \emph{factorizing formal contexts} $(G,F,I_{GF})$ and $(F,M,I_{FM})$ such that
\begin{align*}
  (g,m) \in I \iff (g,f) \in I_{GF} \text{ and } (f,m) \in I_{FM}
\end{align*}
for some $f \in F$.
Elements in $F$ are called \emph{Boolean factors}.
For each factor $f\in F$, let the \emph{factorizing family} be the pair $(A,B)$ with $A = \{g\mid (g,f) \in I_{GF}\}$ and $B = \{m\mid (f,m) \in I_{FM}\}$.
Then, for each object $g\in A$ and each attribute $m \in B$ it holds that  $(g,m) \in I$.
Thus, it is sensible and possible to extend each factor such that its factorizing family is a formal concept.
The emerging factorization is called a \emph{conceptual factorization}.
A Boolean factorization with few factors can have two benefits.
It compresses the size of a dataset and increases its understandability.
However, the existence of a small factorization for a formal context is not guaranteed.
Contrarily, for a given $k$ deciding if a factorization into $k$ factors exists is an $NP$-complete task~\cite{Belohlavek.2010}.

For our running example, there exists a conceptual factorization using 17 factors as depicted in \cref{fig:boolean_faktor}.

\subsection{Ordinal Factor Analysis}
\label{sec:ordin-fact-analys}

To overcome the shortcomings of large Boolean factorizations, Ganter and Glodeanu propose in~\cite{Ganter.2012} to group multiple Boolean factors into a single factor as follows.
For a given formal context $\context=(G,M,I)$ with a pair of factorizing contexts $(G,F,I_{GF})$ and $(F,M,I_{FM})$ a set $E\subseteq F$ is called a \emph{many-valued} factor of $(G,M,I)$.
An \emph{ordinal factor} is then defined as a many-valued ordinal factor $E$ where for all elements $e_1,e_2\in E$ with factorizing families $(A_1, B_1)$ and $(A_2, B_2)$ it holds that $A_1 \subset A_2$ or $A_2 \subset A_1$.
A \emph{complete ordinal factorization} of width $k$ of a formal context $\context=(G,M,I)$ is a set of ordinal factors $F_1,\ldots F_k$ such that $\bigcup_{i=1}^k F_i = I$.
Similar to the case of Boolean factors, it is possible to extend each ordinal factor to a chain of formal concepts.
Such an ordinal factorization is called a \emph{conceptual ordinal factorization}.

\begin{figure}[!hb]
  \centering
  \includegraphics[width=\linewidth]{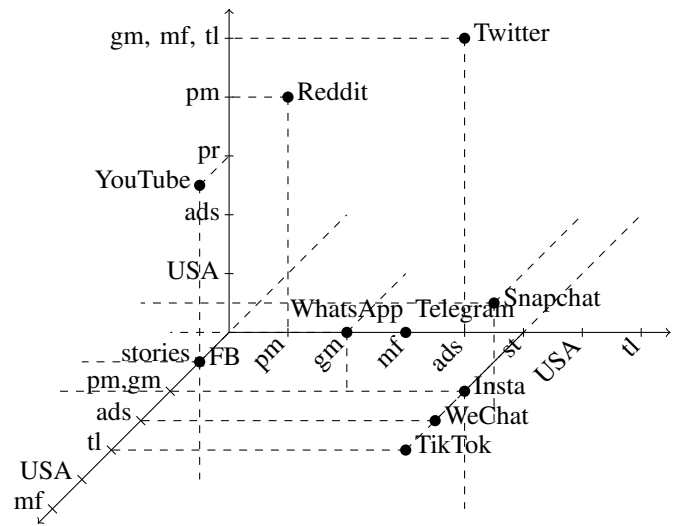}
  \caption{A complete ordinal factorization of the dataset from \cref{fig:dataset} into three factors.}
\end{figure}
Ordinal factors are highly related to Ferrers relations.
A Ferrers relation $F$ in the context $(G,M,I)$ is a subset of $G \times M$ where for all $g,h \in G$ and $m,n \in M$ it holds that
\begin{align*}
  (g,m) \in I \land (h,n)\in I \Rightarrow (g,n) \in I \lor (h,m) \in I.
\end{align*}
We call a Ferrers relation $F$ maximal in a formal context $(G,M,I)$ if $F \subset I$ and there is no Ferrers relation $F'\subset I$ with $F \subsetneq F'$.
The maximal Ferrers relations correspond 1-to-1 to the maximal chains of the concept lattice, which describe the maximal conceptual ordinal factors.
Because of this duality, we refer to both, maximal Ferrers relations as well as maximal chains of formal concepts as \emph{maximal ordinal factors}.
We call the \emph{size of an ordinal factor} the number of incidences contained in its Ferrers relation.

Consider once again the running example.
It is not possible to do a factorization with two factors because no two of the three elements (YouTube, premium), (WhatsApp, mobile-first) and (Facebook, timeline) can appear in the same Ferrers relation.
On the other hand, it is possible to generate a complete ordinal factorization into three factors as follows:

\begin{enumerate}
  \item $f17 < f16 < f13 < f10 < f8 < f3 < f1$
  \item $f15 < f12 < f6 < f5 < f2$
  \item $f14 < f11 < f9 < f7 < f4 < f1$
\end{enumerate}

Note that every factor from the conceptual Boolean factorization in \cref{fig:boolean_faktor} appears in at least one ordinal factor, and this is thus a complete ordinal factorization.

Ganter and Glodeau propose a coordinate system to depict the information of the dataset using ordinal factorizations.
Every axis of this plot represents a single ordinal factor as follows.
Each tick of such an axis is associated with a concept of the concept chain.
The label for the tick contains the additional attributes that this concept gains compared to the previous tick's.
The plot depicts the objects in each axis such that they are one tick before the first tick with an attribute they do not have.
For our given factorization of the running example, the complete factorization would result in a three-dimensional coordinate system.
Presenting a more complex dataset with three ordinal factors on a plane results in hard to read projections as \cref{fig:3d} demonstrates.
Thus, restricting such a plot to the two largest ordinal factors is sensible.
However, with this method, the information of the third factor is lost.
\Cref{fig:3d} compares an ordinal factor analysis plot to a principal component analysis one.

\section{Greedy Ordinal Factorizations are Hard}
\label{sec:comp-greedy-ordin}

In the following two sections, we introduce and investigate greedy ordinal factorizations.
\begin{definition}
  Let $(G,M,I)$ be a formal context.
  A set of ordinal factors $F_1,\ldots,F_k$ is called greedy, if for each $i \in \{1,\ldots,k\}$ there is no factor $\tilde{F_i}$ with $|\tilde{F_i}\setminus\{F_1 \cup \cdots \cup F_{i-1}\}|> |F_i\setminus\{F_1 \cup \cdots \cup F_{i-1}\}|$.
\end{definition}
Repeatedly computing a Ferrer's relation, which covers the maximum possible uncovered part of the incidence, results in such a factorization.
A complete ordinal factorization arises by repeating this process until every part of the incidence relation is covered.
Still, it turns out that even deciding on the size of the first ordinal factor of this greedy process is an $NP$-complete problem.
To see this, consider an auxiliary Lemma from Yannakakis~\cite{Yannakakis.1981}.
This Lemma refers to the bipartite graphs corresponding to Ferrer's relations as chain graphs.
The to notions of this work adapted version is the following.
\begin{lemma}
  Let $\context=(G,M,I)$ be a formal context and $k$ a natural number.
  It is $NP$-complete to decide whether there is a set $\tilde{I} \subset G \times M$ of size $k$  such that $I\cup \tilde{I}$ is a Ferrers relation.
\end{lemma}
The complement of a Ferrers relation is once again a Ferrers relation.
Thus, it follows immediately that deciding on the number of pairs that have to be removed from a binary relation to make it a Ferrers relation is equally computationally complex.
\begin{corollary}
  Let $\context=(G,M,I)$ be a formal context and $k$ a natural number.
  It is $NP$-complete to decide whether there is a set $\tilde{I} \subset I$ of size $k$  such that $I\setminus \tilde{I}$ is a Ferrers relation.
\end{corollary}
\begin{proof}
  To see this, one has to show that the complement of a Ferrer's relation is once again a Ferrer's relation.
  Assume not, i.e. there is a Ferrer's relation  $F \subset G \times M$ such that its complement $(G\times M)\setminus F$ is no Ferrer's relation.
  Then there has to be $g,h \in G$ and $m,n\in M$ with $(g,m) \not\in F$ and $(h,n)\not\in F$ for which $(g,n) \in F$ and $ (h,m) \in F$, a contradiction to $F$ being a Ferrer's relation.
  Now let $(G,M,I)$ be a context and $\tilde{I} \subset G \times M$.
  Then $(G\times M)\setminus (I \cup \tilde{I})$ is a Ferrers relation exactly when $I\setminus \tilde{I}$ also is a Ferrers relation, which completes the proof.
\end{proof}

Thus, even deciding if the first greedy factor has a given size $k$ is an $NP$-complete task which makes the problem of computing a greedy ordinal factorization $NP$-hard.

\section{Computing Greedy Ordinal Factorizations}
In this section, we provide \textsc{OrdiFIND} (Algorithm for \textsc{Ord}inal \textsc{F}actors \textsc{in} Binary \textsc{D}ata), an algorithm that to compute greedy ordinal factorization.
As discussed in the last section, this problem is of high computation complexity.
Thus, the best algorithm we can hope to find will still have exponential time with respect to the input size.
Still, it is possible to leverage fast concept lattice algorithms from formal concept analysis.
Especially Lindig's algorithm~\cite{Lindig.2000} which can be equipped with speed-up techniques~\cite{Krajca.2021} is useful.
This algorithm computes the covering relation of a formal context in a reasonable time.
The main advantage of this approach is that the exponential-time task of computing the concept lattice is executed only once.
After this, the algorithm has linear time complexity in the size of the covering relation of the concept lattice.
Note that the size of this covering relation may still be exponential in the size of the input data.

\begin{algorithm}[t]
  \raggedright
  \begin{tabular}{ll}
    \textbf{Input:}  & Concept lattice with covering relation $({\mathfrak B},\prec)$ \\
                     & Covered $E$                                                    \\
    \textbf{Output:} & Maximal Ferrers Relation $F$
  \end{tabular}
  \vspace*{0.3em}
  \hrule
  \begin{lstlisting}
$F_{(A,B)} \coloneqq \emptyset$ $\forall (A,B) \in {\mathfrak B}$
$L \coloneqq linear\_extension({\mathfrak B},\prec)$
for $(A_1,B_1)$ in $L$:
    for $(A_2,B_2) \in {\mathfrak B}$ with $(A_2,B_2) \prec (A_1,B_1)$:
        $\widetilde{F_{(A_1,B_1)}} = F_{(A_2,B_2)} \cup (A_1 \times B_1)$
        if $|\widetilde{F_{(A_1,B_1)}} \setminus E|\geq|F_{(A_1,B_1)} \setminus E|$:
            $F_{(A_1,B_1)} = \widetilde{F_{(A_1,B_1)}}$
return $F_{\max({\mathfrak B},\prec )}$
\end{lstlisting}
  \caption{Maximal Ferrers Relation}
  \label{alg1}
\end{algorithm}

The routine described in \Cref{alg1} lays the foundation for \textsc{OrdiFIND}.
It takes the covering relation of a concept lattice and a set  $E$ as its input.
The routine then computes an ordinal factor that covers a maximal number of incidences in  $I\setminus E$.
Thus, the output of the routine is an ordinal factor $F$, in its Ferrers relation form, such that there is no ordinal factor $\tilde{F}$ with $|\tilde{F} \setminus E| > |F\setminus E|$.
The main idea of the routine is the following.
The concept lattice of a formal context is iterated from the bottom to the top in order of a linear extension, i.e., in a linear order compatible with the concept lattice.
Thus, each concept is not processed before the same process finishes for all smaller ones.
After each concept is processed, the ordinal factor covering a maximal number of incidences of $I\setminus E$ with this specific concept $(A, B)$ as a top concept is in the variable $F_{(A, B)}$.
Therefore, after every element is processed, the variable corresponds to the top concept $F_{\max (\mathfrak B, \prec)}$ contains the factor excluding with the highest number of incidences from $I\setminus E$.
For a set of already computed factors, it is possible to compute the ordinal factor containing a maximal number of uncovered incidences by choosing the union of those factors as $E$.
Iteratively repeating this procedure thus results in a complete greedy factorization of the formal context as described in \cref{alg2}.
\begin{algorithm}[b]
  \caption{Naive Complete Ordinal Factorization}
  \label{alg2}
  \raggedright
  \begin{tabular}{ll}
    \textbf{Input:}  & Formal context $(G,M,I)$                      \\
                     & Covering relation $({\mathfrak B},\prec)$     \\
    \textbf{Output:} & Greedy ordinal factorization $F_1,\ldots,F_k$
  \end{tabular}
  \vspace*{0.3em}
  \hrule

  \begin{lstlisting}
$E = \emptyset$
$i = 1$
while $E \neq I$:
    $F_i = max\_ferrers(({\mathfrak B},\prec),E)$
    $E = E \cup F_i$
    $i = i+1$
return $F_1,\ldots ,F_i$
\end{lstlisting}
\end{algorithm}

Consequently, the following lemma concerning the overall runtime of the algorithm follows.
\begin{lemma}
  For a formal context with $k$ concepts and given covering relation, it is possible to compute a greedy ordinal factorization consisting of $r$ ordinal factors in $O(rk^2)$.
\end{lemma}

\begin{proof}
  The runtime of \cref{alg1} is bounded from above by the size of the covering relation.
  The covering relation is bounded from above by the number of concept pairs.
  The loop in \cref{alg2} is repeated $r$ times.
\end{proof}

The algorithm described above computes a greedy ordinal factorization.
For the usefulness of this algorithm, it is crucial to investigate how good the factorization is, i.e., how it performs compared to an ordinal factorization that minimizes the number of factors.
We investigate this in the following Lemma.

\begin{lemma}
  Let $\context=(G,M,I)$ that can be optimally decomposed into $k$ ordinal factors.
  Let $F_1,\ldots,F_r$ be a greedy ordinal factor decomposition.
  Then $r \leq k \log_e |I|$.
\end{lemma}

\begin{proof}
  By pigeon-hole principle it holds that $|F_i|\geq (|I|-|F_1 \cup \ldots \cup F_{i-1}|)/k$
  for all $i \in \mathbb{N}$. Thus, for $r>k \log_e |I|$ it holds that
  $|I|-|F_1|-\cdots-|F_{r}|\leq |I|(1-\frac{1}{k})^r\leq |I|e^{-r/k} < 1$,
  implying that all incidences are covered after $r$ repetitions.
\end{proof}

The naive algorithm, proposed in \cref{alg1,alg2}, reprocesses each concept in every step for each computation of a new factor.
However, especially in the later and smaller factors, the algorithm does not modify all concepts in the lattice.
Identifying concepts not touched by the algorithm enables a speed-up technique as the value of those can be saved between the steps.
To do so, we have to modify how the naive algorithm stores the computed information.
Instead of saving all incidences of the ordinal factor at each concept in the lattice, we only store their size.
Additionally, we store for each concept in the factor its predecessor.
The factor is then extracted by trailing the predecessors, starting with the top concept.
Now, not every factor contains every attribute.
Thus, after a factor with missing attributes was computed, only concepts that have a greater concept with an attribute in it have to be recomputed.
The speed speed-up thus emerges from just recomputing those concepts.
To determine which concepts are to recompute, we store the attributes of the factor in some set $R$.
We use the set $R$ furthermore to determine which concepts to iterate next.
Instead of enumerating the concepts in the order of a previously fixed linear extension, the smallest attribute set of $R$ is always the next one to iterate.
Then it is removed from $R$.
To enable this extraction, we implement $R$ as a heap.
The size of the attribute sets of concepts increases for smaller ones
Thus, a linear order compatible with the concept lattice is respected.
The proposed algorithm \textsc{OrdiFIND} in \cref{alg:faster} uses the new ideas from this paragraph to speed up the first two algorithms.
\begin{algorithm}[t]
  \caption{OrdiFIND}
  \label{alg:faster}
  \raggedright
  \begin{tabular}{ll}
    \textbf{Input:}  & Formal context $(G,M,I)$                      \\
                     & Covering relation $({\mathfrak B},\prec)$     \\
    \textbf{Output:} & Greedy ordinal factorization $F_1,\ldots,F_k$
  \end{tabular}
  \vspace*{0.3em}
  \hrule

  \begin{lstlisting}
$E = \emptyset$
$i = 1$
$R = {\min({\mathfrak B} \prec)}$
while $E \neq I$:
    $b_{(A,B)} \coloneqq \bot$ $\forall (A,B) \in {\mathfrak B}$
    $s_{(A,B)} \coloneqq \emptyset$ $\forall (A,B) \in {\mathfrak B}$
    while $R \neq \emptyset$:
        $(A_1,B_1) = R.pop\_concept\_min\_no\_attr()$
        for $(A_2,B_2) \in {\mathfrak B}$ with $(A_1,B_1) \prec (A_2,B_2)$:
            $\widetilde{s_{(A_2,B_2)}} = s_{(A_1,B_1)} + |(A_2 \times B_2)\setminus E|$
            if $\widetilde{s_{(A_2,B_2)}} > s_{(A_2,B_2)}$:
                $s_{(A_2,B_2)} = \widetilde{s_{(A_2,B_2)}}$
                $b_{(A_2,B_2)} = (A_1,B_1)$
                $R = R \cup \{(A_2,B_2)\}$
    $F_i \coloneqq \emptyset$
    $b = {\max({\mathfrak B} \prec)}$
    while $b \neq \bot$:
        $F_i$ = $F_i \cup \{(A,B)\}$
        $b = b_{(A,B)}$
    $R = \{(g'',g') \mid g \in B, (A,B) \in F_i\}$
    $i = i + 1$
return $F_1,\ldots ,F_i$
\end{lstlisting}
\end{algorithm}

\section{Relationship between Factors and Objects}

Let $(G,M,I)$ be a formal context with an ordinal factorization $F_1, \ldots F_k$.
For a fixed factor $F_i=(A_1,B_1),\ldots,(A_j,B_j)$ and an object $g$ call the position of $g$ in $F_i$  the maximal $r$, such that $B_r \subset g'$.
Because $F_i$ is a chain of concepts, all attribute sets of the concepts $B_1\ldots B_r$ are thus subsets of $g'$.
If this is true, we say that the object $g$ \emph{supports the factor} $F_i$ until position $r$.
If $r = j$ the object $g$ \emph{supports the whole factor} $F_i$.

For two objects $g_1, g_2 \in G$ in a dataset, we can now define the \emph{ordinal factorization distance} as
\begin{align*}
  d(g_1, g_2) = |g_1'\setminus g_2'|.
\end{align*}
This distance between two objects thus counts the number of attributes that one object has, and the other one is missing.
Note that this distance is not symmetric and f $g_1' \subset g_2'$ it becomes 0.
The Hamming distance, which is the symmetric version of the distance defined above is commonly used to compute distances in binary datasets and can then be defined by
\begin{align*}
  d_h(g_1,g_2) = d(g_1,g_2)+d(g_2,g_1).
\end{align*}

We can generalize the notion of the ordinal factorization distance to positions in the factors.
To do so, we compute the number of attributes in a factorization based on these positions.
Then the number of missing attributes of the object is computed.
Let thus $(G,M,I)$ be a formal context with a factorization $F_1,\ldots, F_k$ and $p_1,\ldots,p_k$ integers representing the position in each factor.
Define the \emph{ordinal factorization distance of the object $g$ to the positions $p_1,\ldots,p_k$} as
\begin{align*}
  d_F(g,p_1,\ldots,p_k) = |\{m \mid m \in F_i \text{ before position $p_i$ }\}\setminus g'|.
\end{align*}
The distance thus counts how many attributes an item is missing support all factors until position $F_i$ until position $p_i$.
Note that it is possible to which the \emph{ordinal factorization distance} of an object is 0.
This is trivially true, if for $p_1 = \cdots = p_k = 0$.

\section{Ordinal Factorization as a Tool to Discover Structure}
\label{demo}
We see the ordinal factor analysis as a tool to explain and discover relationships in unordered datasets.
The computationally expensive task to compute the factorization only has to be done once in the preprocessing step.
To make the tool easily accessible, we envision a web platform that only has to do the relatively light computation of positions of objects in the ordinal factors.
A user can then interactively navigate the dataset to discover the preexisting relationships in the data.
To do so, we propose a navigation system consisting of slide controls.
Each slide control thereby corresponds to one factor.
In the case of an ordinal factorization with $k$ factors, there are $k$ slide controls.
For each concept of a factor, there is a position in the corresponding slide control.
As the top and bottom concepts of a factor usually have an empty object or attribute set, they do not get a corresponding position in the slide control.
Additionally, a 0-position is introduced for each slide control to indicate an empty selection.
Each position is annotated with the attributes that the corresponding concept gains compared to the previous one in the factor.
An analyst can then use the sliders to select the attribute positions for the slide controls.
The system then displays the objects ordered by their ordinal factorization distance to the selected position.
The second feature of this platform is that the analyst can select an object by clicking it.
Then, the system automatically moves the slider for each factor to the maximal position the selected object supports.
The remaining objects are then sorted based on the ordinal factorization distance of the selected position.
For a prototype of how such a navigation system might work, we refer the reader to our demonstration platform\footnote{\url{https://factoranalysis.github.io/ordinal/}}.
We computed the ordinal factorizations of this tool using the greedy factor analysis as proposed in this work.
The source code of this platform may be the basis for a more sophisticated exploration tool in the future.

\begin{table*}[t]
  \centering
  \caption{Dataset descriptions}
  \label{tab:datasets}
  \begin{tabular}{lrrrrr}
    \toprule
                                     & German AI & General AI & IMDb     & Wahl-O-Mat parties & Wahl-O-Mat statements \\
    \midrule
    Attributes                       & 133       & 137        & 28       & 38                 & 76                    \\
    Objects                          & 2238      & 218308     & 3784644  & 76                 & 38                    \\
    Concepts                         & 12709     & 1981691    & 6765616  & 24245              & 24245                 \\
    Density                          & 0.0258    & 0.0149     & 0.0680   & 0.429              & 0.429                 \\
    Mean attributes per concept      & 4.87      & 7.24       & 11.52    & 7.26               & 8.28                  \\
    Mean object per concept          & 7.35      & 13.54      & 792.78   & 8.28               & 7.26                  \\
    Size of covering relation        & 46776     & 11455411   & 63990855 & 24245              & 24245                 \\
    Number of greedy ordinal factors & 129       & 137        & 28       & 20                 & 20                    \\
    \bottomrule
  \end{tabular}
\end{table*}

\begin{figure*}[!ht]
  \centering
  \begin{enumerate}
    \item ICRA - IROS - I. J. Robotics Res. - IEEE Trans. Robotics - Robotics and Autonomous Systems - Auton. Robots - IEEE Robotics and Automation Letters - IEEE Robot. Automat. Mag. - J. Field Robotics - Robotics: Science and Systems - FSR - ISRR - AAAI - IJCAI - NIPS, ICML - Humanoids, Artif. Intell., J. Artif. Intell. Res., Machine Learning - ACL, AAMAS, CoRL, IEEE Trans. Pattern Anal. Mach. Intell., ICCV, UAI, J. Mach. Learn. Res., ICAPS
    \item CVPR - ICCV - ECCV - IEEE Trans. Pattern Anal. Mach. Intell. - International Journal of Computer Vision - Computer Vision and Image Understanding - ACCV - AAAI - IJCAI - NIPS - ICML - IEEE Trans. Neural Networks - NeurIPS - IJCNN - IEEE Trans. Knowl. Data Eng., ICDM - SDM, KDD, AISTATS - IEEE Trans. Neural Netw. Learning Syst., IEEE Trans. Systems, Man, and Cybernetics, Part C, Neural Computation - Data Min. Knowl. Discov., NAACL-HLT, Machine Learning, IROS, UAI, J. Mach. Learn. Res.
    \item AAAI - IJCAI - CIKM - KDD - ICDM - IEEE Trans. Knowl. Data Eng. - SDM - PAKDD - WWW - WSDM - SIGIR - ECML/PKDD - ICML - ACL - NIPS - CVPR - UAI, IEEE Trans. Pattern Anal. Mach. Intell. - Artif. Intell. - Machine Learning - ACL/IJCNLP, International Semantic Web Conference, EMNLP/IJCNLP, ECML - NeurIPS, HLT-NAACL, ICRA, CoNLL, Autonomous Agents and Multi-Agent Systems, NAACL-HLT, PKDD, ML, Robotics and Autonomous Systems, ECAI, COLT, AISTATS, J. Mach. Learn. Res., EMNLP, IROS, COLING, EMNLP-CoNLL
    \item IJCNN - Neural Networks - IEEE Trans. Neural Networks - ICANN - Neural Computation - NIPS - J. Mach. Learn. Res. - IEEE Trans. Pattern Anal. Mach. Intell. - Machine Learning - ICML - ECML/PKDD - UAI - SDM - KDD - COLT - EMNLP/IJCNLP, AAAI, PAKDD, AISTATS, RecSys - HLT-NAACL, CVPR, International Journal of Computer Vision, CIKM, SIGIR, WSDM, EACL, ICCV, ECML, WWW, EMNLP-CoNLL
    \item IROS - Robotics and Autonomous Systems - Auton. Robots - IEEE Robot. Automat. Mag. - ICRA - Humanoids - ISRR - Robotics: Science and Systems - IEEE Robotics and Automation Letters - I. J. Robotics Res. - NIPS - IJCAI - AISTATS - J. Mach. Learn. Res. - ICML, CoRL, ECML, IEEE Trans. Robotics, Neural Networks, Neural Computation, IJCNN - ECML/PKDD, IEEE Trans. Neural Netw. Learning Syst., Artif. Intell., AAAI, ICANN, AAMAS, Machine Learning, ECAI, IEEE Trans. Pattern Anal. Mach. Intell., ICAPS
    \item NIPS - ICML - AISTATS - J. Mach. Learn. Res. - UAI - Machine Learning - COLT - NeurIPS - KDD - ICDM - IJCAI - IEEE Trans. Pattern Anal. Mach. Intell. - AAAI - Neural Computation - ICRA, IEEE Trans. Neural Networks, SDM, Neural Networks, IJCNN - Humanoids, ICANN, Auton. Robots, PAKDD, IROS - Artif. Intell., IEEE Robotics and Automation Letters, CVPR, International Journal of Computer Vision, ACCV, PKDD, I. J. Robotics Res., WSDM, ECCV, ICCV, ECML, WWW, Robotics: Science and Systems, ICLR
    \item CHI - UbiComp - ACM Trans. Comput.-Hum. Interact. - IUI - IJCAI - AAMAS - J. Artif. Intell. Res. - HCOMP, Artif. Intell., AAAI - WWW, UAI - HLT-NAACL, ACL, ICML, CIKM, KDD, WSDM, EMNLP - ACL/IJCNLP, IEEE Trans. Knowl. Data Eng., KR, EMNLP/IJCNLP, NAACL-HLT, Machine Learning, AIPS, ICAPS
    \item IJCAI - Artif. Intell. - ECAI - J. Artif. Intell. Res. - KR - JELIA - TPLP - ICLP - AAAI - LPAR - RR - ESWC - IEEE Trans. Knowl. Data Eng., International Semantic Web Conference, RuleML - RuleML+RR - IJCAR, ILP, CPAIOR, CP, WWW, TARK
    \item CIKM - SIGIR - WWW - WSDM - KDD - ACL - EMNLP - EMNLP/IJCNLP - AAAI - NAACL-HLT - NeurIPS - NIPS, ICML - IJCAI - ECML/PKDD - J. Artif. Intell. Res. - COLT, ICDM, Machine Learning - ACL/IJCNLP, HLT-NAACL, Artif. Intell., HLT/EMNLP, ML, RecSys, SDM, UAI, EMNLP-CoNLL, ICLR - CoNLL, PAKDD, EACL, ICCV, COLING, J. Mach. Learn. Res.
    \item ACL - EMNLP - COLING - HLT-NAACL - EACL - NAACL-HLT - EMNLP/IJCNLP - CoNLL - EMNLP-CoNLL - HLT/EMNLP - J. Artif. Intell. Res. - AAAI - ICML - IJCAI, KDD, Machine Learning, NIPS, J. Mach. Learn. Res. - COLING-ACL - IEEE Trans. Knowl. Data Eng., KR, ILP, AAMAS, PKDD, ECCV, WWW, ECML/PKDD, CVPR, CIKM, IEEE Trans. Pattern Anal. Mach. Intell., NeurIPS, Artif. Intell., ECML, Neural Computation, COLT, AISTATS, ICDM, J. Autom. Reasoning
    \item KDD - ICDM - SDM - IEEE Trans. Knowl. Data Eng. - PAKDD - Data Min. Knowl. Discov. - ECML/PKDD - PKDD - Machine Learning - ICML - IJCAI - CIKM - AAAI - J. Mach. Learn. Res. - UAI - NeurIPS, J. Artif. Intell. Res., ECML - KR, Artif. Intell., ILP, ICRA, Int. J. Approx. Reasoning, Auton. Robots, ICLP, TPLP, EWSL, ML, ECAI, Inductive Logic Programming Workshop, Computer Vision and Image Understanding, ICCV, IROS, EMNLP-CoNLL, Probabilistic Graphical Models, IJCNN
    \item IEEE Trans. Pattern Anal. Mach. Intell. - ECCV - ACCV - Computer Vision and Image Understanding - International Journal of Computer Vision - ICCV - CVPR - IEEE Trans. Neural Netw. Learning Syst. - AAAI - IEEE Trans. Knowl. Data Eng. - ICDM - WWW - IJCAI - SIGIR - PAKDD - CHI, EMNLP, WSDM, KDD - EMNLP/IJCNLP, ECML/PKDD, HLT-NAACL, ACL, CIKM, SDM
  \end{enumerate}
  \caption{The attributes of the first twelve factors of a greedy factorization for the German AI dataset}
  \label{fig:germanai}
\end{figure*}

\section{Case Study}
\label{sec:experiments}

Now, we demonstrate for five exemplary datasets that ordinal factor analysis is a method that can be applied to discover relationships in a dataset.

\subsection{Source Code}

Our algorithms are implemented in Python 3.
An improved version of Lindigs algorithm~\cite{Lindig.2000, Krajca.2021} directly and fastly computes the covering relation of the concept lattice needed for our approach.
We thus use this method to compute the concept lattice.
Even though we carefully implemented the code of our algorithms, it is focused on allowing experiments rather than speed.
We thus believe that it is possible that optimizing the code could result in a speed-up.
The source code\footnote{\url{https://figshare.com/s/8edc72eaebc53e74f146}} is available for review and reproducibility.%

\subsection{Datasets}
We conduct our experiments on multiple datasets from various sources, which we introduce in this section.

\paragraph{German AI and General AI}
We generated this bibliometry based on a work by Koopmann et.al.~\cite{Koopmann.2021}.
It contains the publication relationship between authors and conferences in the realm of artificial intelligence.
They chose the conferences in this data based on an article by Kersting~\cite{Kersting.2019}.
We generate from this data two different datasets.
The first one contains the AI community restricted to authors with a German affiliation, while the other is unrestricted.
We refer to the two datasets by \emph{German AI} and \emph{General AI}.

\paragraph{IMDb}
We generated this dataset from the IMDb, a database that contains information on films and series.
The movies and series are the objects of the dataset and their categories the attributes.
We use the data as it was on the 6th of October 2021~\cite{imdb} and refer to this dataset as the \emph{IMDb} dataset.

\paragraph{Wahl-O-Mat}
The Wahl-O-Mat is a website provided by the ``Bundeszentrale für politische Bildung''~\cite{wahlomat} in Germany.
It is a tool that helps german citizens decide which parties to vote in federal elections.
Thereby a set of statements is provided on which the citizens have to decide whether they agree or disagree.
Afterward, the system compares the results to statements of different parties to suggest which party might fit the views of the specific citizen.
We extracted the statements and parties of the 2021 election to create two datasets.
The first one, which we refer to as \emph{Wahl-O-Mat-parties} has the parties as objects and the statements as attributes.
The dataset \emph{Wahl-O-Mat-statements} has the statements as objects and the parties as attributes.

\subsection{Runtime}
We performed all computations on an Intel Xeon Gold 5122 CPU equipped with 800 GB of memory.
For each dataset, we used the methods from this work to compute a complete ordinal factorization.
In \cref{tab:runtime} the runtime of \textsc{OrdiFIND} is compared to the naive algorithm.
Also, the runtime of the computation of the concept lattice is listed.
One can observe that in almost all instances, the computation of the covering relation takes a substantial amount of time.
It is also clearly visible that \textsc{OrdiFIND} outperforms the naive version of the algorithm.
We were not able to compute the factorization of the IMDb dataset using the naive algorithm.
All in all, we conclude that it is feasible to apply the method proposed in this paper on all datasets, for which it is possible to compute the concept lattice.
\begin{table}[b]
  \centering
  \caption{Runtime of the algorithm}
  \label{tab:runtime}
  \begin{tabular}{lrrr}
    \toprule
                          & Concepts & Naive  & OrdiFIND \\
    \midrule
    General AI            & 4.63h    & 4.18h  & 3.57h    \\
    German AI             & 3.76s    & 33.54s & 12.97s   \\
    IMDb                  & 4.46d    & -      & 14.42d   \\
    Wahl-O-Mat parties    & 4.34s    & 16.92s & 11.96s   \\
    Wahl-O-Mat statements & 8.18s    & 14.23s & 13.36s   \\
    \bottomrule
  \end{tabular}
\end{table}

\subsection{Interesting Findings}
\label{sec:discussion}

In this section we demonstrate how to use the method to discover relationships and discuss findings in our datasets.

\paragraph{German AI and General AI}

The greedy factorization of the German AI and the General AI datasets results in 129 and 137 factors, respectively.
For the German AI dataset, we include the factors in~\cref{fig:germanai}, the prototype web platform from \cref{demo} provides insight for the General AI community.
It seems sensible to interpret the single factors as communities of conferences, where authors of similar interest publish.
For example, the first factor in both cases consists of Robotics conferences.
The second factor of the German AI dataset corresponds then to the computer vision community, while the third factor contains the general AI conferences.
Because of the nature of ordinal factorizations, the conferences that appear early in such a factor are the more general ones.
The ones that are deep into a factor are more specialized.
Such an ordinal factor analysis can thus provide an author with a tool to select a conference.
It is interesting to note that the communities discovered with this method mostly coincide with the AI communities as formulated by Kersting in~\cite{Kersting.2019}.

\paragraph{IMDb}

The complete greedy factorization of the movie dataset by IMDb contains 27 factors.
Each factor corresponds to a linear order of genres.
One can derive knowledge about the movie landscape using these factors as follows.
The first factor seems to focus on light entertainment.
As most movie productions are in this factor, this is one with the highest relevance for the industry.
The second factor then has a more educational focus.
The third one is once again a factor with entertaining shows.
However, the entertainment in this factor is not as light as in the first one.
This factorization is valuable, as it is, to navigate a movie database or any database with a similar genre system.
We furthermore believe that it is possible to extend the current method to a more sophisticated recommendation paradigm based on these factors.
The users could then select, for example, that they want to view a movie that is highly entertaining but also a bit educational.
Then the system would recommend movies based on the ordinal factor distance.

\paragraph{Wahl-O-Mat}

We use two different interpretations of the Wahl-O-Mat dataset for our research.
The reason is that, in our opinion, both datasets are of interest for examination with ordinal factor analysis.
In the first one, the factors contain the political parties, while in the second case, they contain statements.
When discussing political topics, it is common to refer to statements as progressive or conservative.
Often, these statements are even compared with each other concerning their progressiveness or conservativeness.
Thus, the everyday intercourse with political topics implies that there is a linear order on the positions.
The same is true for political parties, which we also ordered similarly.
However, in a diverse political landscape, a single order is most likely not adequate to be analyzed in this way.
In our dataset, we identify parties typically classified as right-wing, with some progressive positions such as raising the minimum wage.
The ordinal factor analysis enables a more sophisticated investigation of this situation.
Each factor provides a linear order which arranges the parties or statements respectively.
Thereby, each factor is consistent with the positions of a party.
If we analyze the statement dataset, we can observe that the classical progressive positions are condensed in the first factor, while the second factor contains the conservative ones.
The other factors are more specialized as the order of the positions bases not on the conservative-progressive spectrum but, for example, migration or economics.
In every factor, the early statements are those many parties hold.
Thus they are the ones that have the highest probability to be implemented by a future government.

\subsection{Limitations}

The main limitation of the method is the computational feasibility.
The algorithm has to solve an $NP$-complete problem which results in an exponential-time algorithm with respect to the size of the input data.
Thus, it is not feasible for large datasets, where it is currently impossible to compute concept lattices.
A further limitation is that the algorithm, as it is, currently can only deal with binary datasets.
Datasets with numerical or categorical data have to be scaled into a binary form before the algorithm can be applied.

\section{Conclusion and Outlook}
\label{sec:conclusion}
In this work, we proposed the algorithm \textsc{OrdiFIND} by building upon the prior existing notion of ordinal factor analysis.
To do so, we leveraged fast algorithms developed in formal concept analysis.
The resulting algorithm enables us to compute a greedy ordinal factorization for larger unordered binary datasets, which commonly appear in data mining tasks.
Leveraging this factorization, we demonstrated how to discover relationships in the original data.

Datasets often consist not only of binary but also already ordinal data.
The ordinal factor analysis in its current form can only deal with this data by interpreting it as binary.
While scaling in formal concept analysis is a tool to deal with this data, factors will not necessarily respect the order encapsulated in the data.
In our opinion, the next step should be to extend this method to deal with ordinal data directly.
Furthermore, faster algorithms for the concept lattice are always desirable, especially in the modern web, where the growth of datasets is often even accelerating.
For datasets of sizes unfeasible for our method, heuristics to approximate the ordinal factors may be developed.
Finally, we believe it could be interesting to explore other non-ordinal many-valued factors.

\bibliographystyle{IEEEtran}
\bibliography{factor}

\clearpage

\end{document}